\documentclass{article} 
\usepackage{iclr2026_malgai,times}

\usepackage{amsmath,amsfonts,bm}

\def\eqref#1{equation~\ref{#1}}

\def\1{\bm{1}}

\DeclareMathAlphabet{\mathsfit}{\encodingdefault}{\sfdefault}{m}{sl}
\SetMathAlphabet{\mathsfit}{bold}{\encodingdefault}{\sfdefault}{bx}{n}

\newcommand{\E}{\mathbb{E}}

\providecommand{\argmax}{}
\renewcommand{\argmax}{\operatornamewithlimits{arg\,max}}

\usepackage{hyperref}
\usepackage{url}

\usepackage{amsmath,amssymb,amsthm}
\usepackage{soul}
\usepackage{array}
\usepackage{tikz}
\usepackage{booktabs}
\usepackage{dsfont}
\usepackage{diagbox}
\usepackage[noabbrev,capitalize]{cleveref}
\usetikzlibrary{calc}

\newtheorem{definition}{Definition}

\newtheorem{lemma}{Lemma}

\newtheorem{corollary}{Corollary}
\newtheorem{conjecture}{Conjecture}
\newtheorem{proposition}{Proposition}

\newcommand{\KLBP}{\mathrm{D}^\beta_\mathrm{KLBP}}
\newcommand{\BRD}{\mathrm{D}_{\mathrm{BR}}}

\newcommand{\abs}[1]{\left|#1\right|}
\newcommand{\bO}[1]{\mathcal{O}\left(#1\right)}
\newcommand{\norm}[1]{\lVert#1\rVert}

\usepackage{booktabs}
\usepackage{colortbl}
\usepackage{xcolor}
\usepackage{array}

\definecolor{headerblue}{HTML}{2C3E6B}
\definecolor{headerlightblue}{HTML}{4A6FA5}
\definecolor{rowgray}{HTML}{F2F4F8}
\definecolor{bordergray}{HTML}{D0D5DD}
\definecolor{greencheck}{HTML}{2E7D32}
\definecolor{redcross}{HTML}{C62828}
\definecolor{amberconj}{HTML}{E65100}
\definecolor{bordergray}{HTML}{B0B8C4}

\usepackage{graphicx}

\newcommand{\prooflink}[1]{\hfill\hyperref[#1]{\textup{\textcolor{blue}{[proof]}}}}

\title{Learning the Preferences of a Learning Agent}

\author{Karim Abdel Sadek\thanks{Equal contribution, alphabetical order.}, Mark Bedaywi\footnotemark[1], Rhys Gould\footnotemark[1], Stuart Russell \\
University of California, Berkeley\\
\texttt{\{karimabdel,mark\_bedaywi,rhys\_gould,russell\}@berkeley.edu}
}

\iclrfinalcopy 
\begin{document}

\maketitle

\crefname{theorem}{thm.}{thms.}
\Crefname{theorem}{Thm.}{Thms.}

\crefname{proposition}{prop.}{Props.}
\Crefname{proposition}{Prop.}{Props.}

\crefname{corollary}{cor.}{cors.}
\Crefname{corollary}{Cor.}{Cors.}

\begin{abstract}
For AI systems to be useful to humans, they must understand and act in accordance with our values and preferences. Since specifying preferences is a hard task, \textit{inverse reinforcement learning} (IRL) aims to develop methods that allow for inferring preferences from observed behavior. However, IRL assumes the human to be approximately optimal. This is a big limitation in cases where the human themselves may be learning to act optimally in an environment. In this paper, we formalize the problem of \textit{learning the preferences of a learning agent}: a predictor observes a learner acting online and tries to infer the underlying reward function being (initially suboptimally) optimized by the learner. We model the learner as either being no-regret, or as converging to an optimal Boltzmann policy over time. In each of these settings, we establish theoretical guarantees for various preference learning algorithms, or otherwise show that such guarantees are impossible.
\end{abstract}

\section{Introduction}

AI systems are human-level or superhuman in many disciplines, from playing games such as Go and chess~\citep{silver2016mastering} to predicting protein structures~\citep{jumper2021highly}. Reinforcement learning (RL) has been the underpinning of many such breakthroughs. In RL, much of the progress has been driven by having clean and well-defined \textit{reward functions}, such as game scores or win/loss outcomes~\citep{mnih2013playing,silver2016mastering}. The ability to optimize for a reward signal has even been argued to be sufficient and necessary to get us all we may ever want out of AI systems~\citep{sutton2018rl, silver2021reward}.

A successful approach to designing effective reward functions has been learning them from observed behavior, known as inverse reinforcement learning (IRL)~\citep{ng2000algorithms, abbeel2004apprenticeship}. In this framework, we observe the actions of an agent in an environment and try to infer the underlying reward function that the agent is optimizing. A similar approach in spirit has allowed large language models to follow human instructions and optimize for humans' preferences~\citep{christiano2017deep, stiennon2020learning, bai2022training}. IRL assumes that the agent of interest acts (approximately) optimally/. Yet, we argue that in practice the agent will likely be learning how to act optimally over the course of data collection, taking suboptimal actions initially.

In this paper, we formalize and address the problem of \textit{learning the preferences of a learning agent}. We model this problem via an online learning formulation: the \textit{learner} (or \textit{human}) selects actions online while receiving rewards from an unknown reward function. In particular, we assume that the learner agent is either a no-regret learner, or is converging to an optimal Boltzmann policy. The \textit{predictor} observes the actions taken by the learner and aims to infer the true reward function at each step. Note that recovering the exact reward function is not possible in IRL~\citep{ng1999policy,ng2000algorithms,abbeel2004apprenticeship,cao2021identifiability}. We formalize different notions of error, which include being able to find a reward function that matches the optimal or the Boltzmann policy of the agent, or more simply minimizing a metric distance between rewards~\citep{gleave2020quantifying, skalse2023starc}. 

Our problem setting captures real issues that arise in a variety of settings. For example, consider a recommendation system that tries to suggest movies to a user who is still learning which movies they like. Such a system will need to infer the user's true preferences while the user themselves is acting suboptimally to discover them (e.g. exploring a range of genres before deciding they enjoy thrillers the most). This is a more pressing issue as we develop increasingly personalized AI assistants which will need to infer our preferences across tasks in which humans may not yet be optimal.

\begin{figure}[htbp]
    \vspace{-0.6cm}
    \centering
    \includegraphics[width=0.7\textwidth]{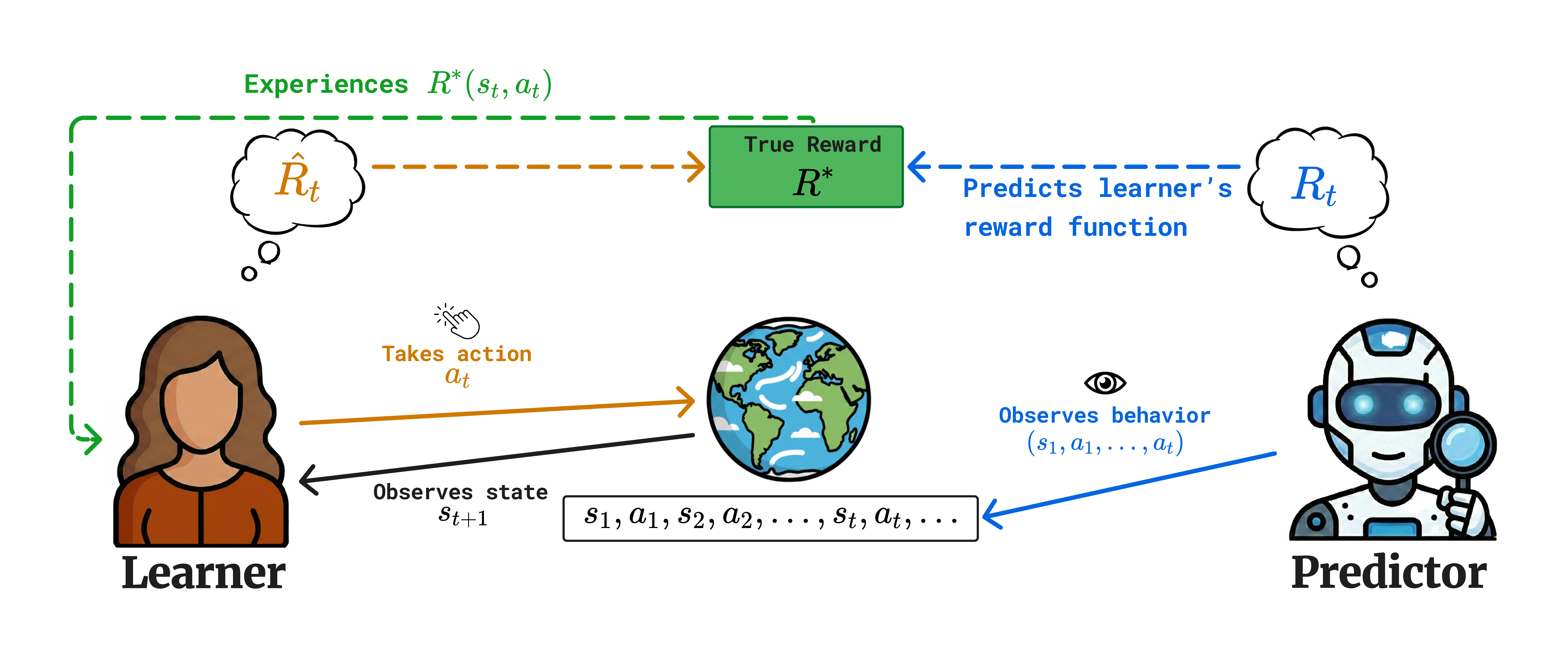}
    \vspace{-.7cm}
    \caption{\textit{Learning the preferences of a learning agent.} A learner interacts with an environment and learns to act optimally over time, with optimality measured by a ground-truth reward function $R^{*}$. The predictor observes only the learner's behavior $(s_1, a_1, \ldots, s_t, a_t)$ and aims to infer the preferences of the agent, producing reward estimates $R_1, \ldots, R_t$ (or $Q$-function estimates).}
    \label{fig:setup}
\end{figure}

\section{Related Work}

\textbf{Inverse Reinforcement Learning} Specifying a reward function may be extremely challenging as AI systems are employed in increasingly complex tasks. Inverse Reinforcement Learning (IRL)~\citep{russell1998learning,ng2000algorithms} studies how a reward function can be recovered from data generated by an expert. However, IRL is underspecified~\citep{jeon2020reward, cao2021identifiability, skalse2023invariance}. An approach to the IRL problem is to use \textit {apprenticeship learning}~\citep{abbeel2004apprenticeship}, where the task is to infer a reward function useful for finding a policy that matches the behavior of a human. \citet{syed2007game} generalizes this approach via a game-theoretic formulation, and shows that it is possible to find a policy that might even outperform the human's. Other relevant work tried to frame IRL as a Bayesian problem~\citep{ramachandran2007bayesian}, or to solely identify the entire set of plausible rewards under different conditions~\citep{metelli2021provably,metelli2023towards}.

\textbf{Modeling human feedback} We are not the first to note that modeling humans correctly is instrumental in correctly identifying the reward function. For example, observing behavior from uncertain humans~\citep{laidlaw2021uncertain} or explicitly learning their suboptimality~\citep{evans2016learning, laidlaw2022boltzmann,reddy2018you} can improve reward inference. All of these works assume the human policy to be fixed over time. \citet{chan2019bandit} is the closest related work. They model the human as learning to play a bandit problem, while being assisted by a robot. Our work differs from theirs in several aspects. First, they only consider the bandit case, while we consider the sequential setting. They also assume the human to be noisily optimal or to be a greedy learner, which are stronger assumption than ours. Additionally, their results only concern assisting the human: there is no explicit learning about the reward function itself, as in our work.

\section{Model and problem setting}

We will consider a learner agent interacting with an environment, and a predictor agent who must determine the ground-truth reward function being optimized by the learner using only \textit{observed behavior}. Our main focus will be proving guarantees for various such prediction strategies. We provide a general overview of the settings we consider in \Cref{tab:comparison}.

\textbf{Background and notation.} In general, we will consider a state space $\mathcal{S}$, an action space $\mathcal{A}$, a ground-truth reward function $R^{*}: \mathcal{S} \times \mathcal{A} \to \mathbb{R}$, a transition distribution $\mu: \mathcal{S} \times \mathcal{A} \to \Delta(\mathcal{S})$, and a discount rate $\gamma \in [0, 1]$, with $\langle S,A,\mu,R^{*},\gamma \rangle$ defining an MDP. For simplicity, in this work we always assume $\gamma = 1$. The \textit{stateless} case is defined by the choice $|\mathcal{S}| = 1$ (letting us write $R^{*}: \mathcal{A} \to \mathbb{R}$), while in the \textit{stateful} case we will generally assume $|\mathcal{S}| > 1$. Throughout the paper we will prove results in both stateless and stateful setups. For the stateful case, we will define the ground-truth action-value function recursively as
$Q^{*}(s, a) := R^{*}(s, a) +  \mathbb{E}_{\mu(s'|s, a)}\left[\max_{a' \in \mathcal{A}}Q^{*}(s', a')\right]$.

We will assume that all learning of the agent takes place within a single episode of finite length. In the stateful case, at time step $t$ the learner will have access to the information $(s_1, a_1, R^{*}(s_1, a_1), \ldots, s_{t-1}, a_{t-1}, R^{*}(s_{t-1}, a_{t-1}), s_t)$ for taking their next action $a_t$. In the stateless case, this reduces to the information $(a_1, R^{*}(a_1), \ldots, a_{t-1}, R^{*}(a_{t-1}))$. We will denote the learner's state-visit count by $N_t(s) := |\{\tau \leq t: s_{\tau} = s\}|$, and $N_t(s, a)$ similarly. The predictor can only observe the learner's \textit{behavior} $(s_1, a_1, \ldots, s_t, a_t)$ and must predict action-value functions $Q_{1:t}$ in the stateful case, or reward functions $R_{1:t}$ in the stateless case. To support the generality of our model, we also show in \Cref{app:robustness_model} that this setup is equivalent to predicting the reward/action-value function only at the end of the episode.

\subsection{Modeling the Learner}

If we wish to prove guarantees about our ability to predict the preferences of a learning agent, we must make some assumptions about the process by which this learner selects actions and learns from their experience. We will consider two models of the learner: (a) assume the agent achieves no-regret (\Cref{subsubsec:lowregret}), or (b) assume the agent behaves Boltzmann-rationally with respect to some estimate of reward/action-value that is updated over time (\Cref{subsubsec:boltzhuman}).

\subsubsection{No-Regret Learner}
\label{subsubsec:lowregret}

One approach to modeling the learner is to assume that they achieve no-regret.

\textbf{Stateless.} In the stateless case, the learner will have taken some actions $(a_1, \ldots, a_T)$, and we say that the learner is $f(T)$-no-regret (for some function $f: \mathbb{N} \to \mathbb{R}$) if
\[
\mathrm{Reg}_L := \max_{a^* }
\sum_{t=1}^T  R^*(a^*)
\;-\;
\sum_{t=1}^T \ R^*(a_t)
\;\leq\; f(T).
\label{eq:noregret_stateless}
\]

where $\mathrm{Reg}_L $ is the regret incurred by the learner. We say that the learner is \textit{no-regret} if they are $f(T)$-no-regret for some $f(T) = o(T)$, i.e. if the average per-step regret vanishes as $T \to \infty$.

\textbf{Stateful.} For the stateful case, we will assume that the no-regret learner selects actions by sampling from a Markovian policy $\hat{p}_t(a_t|s_t)$ that satisfies
$$\mathrm{Reg}_L := \max_{\pi^{*}} \mathbb{E}_{\mu, \pi^{*}}\left[\sum_{t=1}^{T} R^{*}(s_t, a_t)\right] - \mathbb{E}_{\mu, \hat{p}_{1:T}}\left[\sum_{t=1}^{T} R^{*}(s_t, a_t)\right] \leq f(T)$$
for some $f(T) = o(T)$. Note that the time indexing of $\hat{p}_t$ allows for capturing the fact that the agent is learning over time. This notion of regret that replays the episode from the beginning when comparing against the optimal policy in hindsight---as opposed to external regret which appends the fixed action to the current trajectory---is known as policy regret \citep{arora2012online}.

\subsubsection{Boltzmann Rational Learner}
\label{subsubsec:boltzhuman}

We will also consider another approach to modeling the learner: assume that, over time, they converge to a Boltzmann rational policy~\citep{luce1959individual,ziebart2010modeling}.

\textbf{Stateless.} In the stateless case, we model the learner's action selection process as
$a_t \sim \hat{p}_t$, where $\hat{p}_t(a) \propto \exp(\beta \hat{R}_t(a))$
for a rationality parameter $\beta \in [0,\infty)$, and where $\hat{R}_t$ models the learner's estimate of the reward function at time step $t$. To capture the fact that the agent is learning, we assume that $\hat{R}_{t}$ converges to the true reward function $R^{*}$ over time. Formally, 
\begin{equation}
\label{eqn:boltzmannhumanlearning}
\sum_{t=1}^{T} \norm{\hat{R}_t - R^{*}}_{\infty} \leq f(T)
\end{equation}
for some $f(T) = O(T^{\alpha})$ where $\alpha  \in (0, 1)$ captures the rate of the agent's learning.

\textbf{Stateful.} In the stateful case, the learner receives $s_t$ at time step $t$ and selects their action as
$a_t \sim \hat{p}_{t}(\cdot|s_t)$, with $ \hat{p}_{t}(a|s) \propto \exp(\beta \hat{Q}_{t}(s, a))$
where $\hat{Q}_t$ models the learner's estimate of the action-value function at time step $t$. We model learning by assuming that this estimate $\hat{Q}_t$ satisfies:
\begin{equation}
\label{eqn:statefulboltzmannhumanlearning}
\sum_{\tau=1: \, s_{\tau} = s}^{T} \norm{\hat{Q}_{\tau}(s, \cdot) - Q^{*}(s, \cdot)}_{\infty} \leq f(N_T(s))
\end{equation}
for every $s \in \mathcal{S}$ and for some $f(N_T(s)) = O(N_T(s)^{\alpha})$ where $\alpha \in (0, 1)$.

An important detail here is that Boltzmann rationality is not itself a model of learning or exploration. Instead, it is a model capturing the inability to optimize the ground-truth reward function. Boltzmann rationality could model exploration if $\beta \to \infty$ over time. This is the reason why the above is paired with the assumption that the learner's estimate of the reward/action-value converges over time.

\subsection{Evaluation Measures}

There are various reasonable performance measures that we could use to evaluate the quality of the predictor's estimates $R_{1:T} = (R_1, \ldots, R_T)$, (or $Q_{1:T} = (Q_1, \ldots, Q_T)$) for the true reward function $R^{*}$ (or $Q^{*}$). We will introduce the best-response distance (\Cref{subsubsec:bestdist}), the KL divergence between Boltzmann rational policies (\Cref{subsubsec:klbp}), and norm-based measures (\Cref{subsubsec:metnorms}). These measures range from weak (best-response, which only requires matching the optimal action) to strong (norm-based, which require matching the full reward/action-value structure).

\subsubsection{Best-response distance}
\label{subsubsec:bestdist}

Intuitively, the best response distance between two reward functions measures whether both reward functions induce similar optimal policies. We will now make this concrete.

\textbf{Stateless.} Let $a^{R_t}$ denote the best action response given the reward estimate $R_t$. The stateless best-response distance between $R^{*}$ and a sequence of reward functions $R_{1:T} = (R_1, \ldots, R_T)$ is defined as:
$$\BRD(R^{*}, R_{1:T}) := \max_{a^{*} \in \mathcal{A}} \sum_{t=1}^{T} (R^{*}(a^{*}) - R^{*}(a^{R_t}))$$
This measures whether the optimal policies associated with the sequence $R_{1:T}$ also perform well under $R^{*}$, or in other words, whether they match the optimal policy $\pi^{*}$.

\textbf{Stateful.} Let $\pi_{1:T}$ denote the best response policies of the action-value estimates $Q_{1:T}$ at each respective time step $t$.
We define the stateful best-response distance between $Q^{*}$ and $Q_{1:T}$ as:
$$\BRD(Q^{*}, Q_{1:T}) := \max_{\pi^{*}} \mathbb{E}_{\mu, \pi^{*}}\left[\sum_{t=1}^{T} R^{*}(s_{t}, a_{t})\right] - \mathbb{E}_{\mu, \pi_{1:T}} \left[\sum_{t=1}^{T} R^{*}(s_t, a_t)\right]$$

\subsubsection{KL divergence between Boltzmann rational policies}
\label{subsubsec:klbp}

\textbf{Stateless.} Given a reward function $R$, let $\pi^\beta(R)$ denote the associated Boltzmann rational policy, i.e. $\pi^{\beta}(R)(a) \propto \exp(\beta R(a))$. We can define the KL distance between $R^{*}$ and the sequence $R_{1:T}$:
\begin{equation*}
    \KLBP(R^*, R_{1:T}) := \sum_{t=1}^{T}\textrm{KL}(\pi^\beta(R_t) \mid\mid \pi^\beta(R^*)).
\end{equation*}
Note that this direction for the KL penalizes the predictor's induced policy for placing high probability on actions that are suboptimal under $R^{*}$.

\textbf{Stateful.} For the stateful case, we will similarly denote the Boltzmann rational policy associated with the action-value function $Q$ by $\pi^{\beta}(Q)$, with $\pi^{\beta}(Q)(a|s) \propto \exp(\beta Q(s, a))$. We can then define the KL distance between $Q^{*}$ and the sequence $Q_{1:T}$ as:
\begin{equation*}
    \KLBP(Q^*, Q_{1:T}) := \sum_{t=1}^{T}\sum_{s \in \mathcal{S}} v_t(s) \, \textrm{KL}(\pi^\beta(Q_t)(\cdot|s) \mid\mid \pi^\beta(Q^*)(\cdot|s)).
\end{equation*}
for some appropriate state-wise weighting function $v_t$, e.g. $v_t(s) = \sqrt{N_{t-1}(s)/(t-1)}$. We argue that weighting by visit counts is fair; if the learner has only rarely entered a state $s$, we have less data available for predicting $Q^{*}(s,\cdot)$, and so it is fair to suppress the associated prediction errors.

\subsubsection{Norm-based measures}
\label{subsubsec:metnorms}
We also consider norm-based measures of error. In particular, we will consider defining distances between reward functions based on the $\ell_2$ and $\ell_\infty$ norms.

\textbf{Stateless.} Given that for a finite action space we can view a reward function as a real-valued vector $R \in \mathbb{R}^{|\mathcal{A}|}$, we can define the norms $\norm{R}_2 := \sqrt{\sum_{a \in \mathcal{A}} R(a)^2}$ and $\norm{R}_{\infty} := \max_{a \in \mathcal{A}} |R(a)|$. We can then define the associated measures:
$$D_{\ell_{2}}(R^{*}, R_{1:T}) := \sum_{t=1}^{T} \norm{R^{*} - R_t}_{2}, \qquad D_{\ell_{\infty}}(R^{*}, R_{1:T}) := \sum_{t=1}^{T} \norm{R^{*}-R_t}_{\infty}$$

\textbf{Stateful.} The stateful case follows similarly, but we must take an appropriate weighting over states. We define:
$$D_{\ell_{\infty}}(Q^{*}, Q_{1:T}) := \sum_{t=1}^{T} \sum_{s \in \mathcal{S}} v_t(s) \, \norm{Q^{*}(s, \cdot) - Q_t(s, \cdot)}_{\infty}$$
for an appropriate state-wise weighting $v_t$. For example, in \Cref{subsubsec:boltzellinf} we establish performance guarantees for an empirical prediction strategy under the performance measure $D_{\ell_{\infty}}$ with weighting $v_t(s) = \sqrt{N_{t-1}(s)/(t-1)}$.

\section{Theoretical Results}

\begin{table}[h]
\centering
\renewcommand{\arraystretch}{1.5}
\resizebox{\textwidth}{!}{%
\newcolumntype{C}[1]{>{\centering\arraybackslash}m{#1}}
\begin{tabular}{@{} C{1.8cm} !{\color{bordergray}\vrule width 0.8pt} C{2.2cm} C{2.2cm} C{2.2cm} !{\color{bordergray}\vrule width 0.8pt} C{2.2cm} C{2.2cm} C{2.2cm} !{\color{bordergray}\vrule width 0.8pt} @{}}
\toprule[1.5pt]
\textbf{Learner}
& \multicolumn{3}{c!{\color{bordergray}\vrule width 0.8pt}}{\cellcolor{headerblue}\textcolor{white}{\textbf{No-Regret Learner}~$o(T)$}}
& \multicolumn{3}{c!{\color{bordergray}\vrule width 0.8pt}}{\cellcolor{headerlightblue}\textcolor{white}{\textbf{Boltzmann-Rational Learner}}} \\
\cmidrule(lr){1-1} \cmidrule(lr){2-4} \cmidrule(lr){5-7}
\textbf{Metric}
& \textbf{BR}
& \textbf{KL}
& $\boldsymbol{\ell_{\infty}}$~\textbf{norm}
& \textbf{BR}
& \textbf{KL}
& $\boldsymbol{\ell_{\infty}}$~\textbf{norm} \\
\midrule[0.8pt]
\rowcolor{rowgray}
\textbf{Stateful}
& \textcolor{greencheck}{\checkmark}~\Cref{prop:stateful_easy_noregret}
& \textcolor{redcross}{$\boldsymbol{\times}$}~{(\Cref{prop:kl_imp_noregret_human})}
& \textcolor{redcross}{$\boldsymbol{\times}$}~{(\Cref{prop:kl_imp_noregret_human_infty})}
& \textcolor{amberconj}{\textbf{?}}~Conj.~\ref{conj:statefulklbp}
& \textcolor{amberconj}{\textbf{?}}~Conj.~\ref{conj:statefulklbp}
& \textcolor{greencheck}{\checkmark}~\Cref{prop:stateful_linfty_empirical_average} \\
\addlinespace[5pt]
\textbf{Stateless}
& \textcolor{greencheck}{\checkmark}~\Cref{prop:stateless_easy_noregret}
& \textcolor{redcross}{$\boldsymbol{\times}$}~{(\Cref{prop:kl_imp_noregret_human})}
& \textcolor{redcross}{$\boldsymbol{\times}$}~{(\Cref{prop:kl_imp_noregret_human_infty})}
& \textcolor{amberconj}{\textbf{?}}~Conj.~\ref{conj:statelessklbp}
& \textcolor{amberconj}{\textbf{?}}~Conj.~\ref{conj:statelessklbp}
& \textcolor{greencheck}{\checkmark}~\Cref{prop:stateless_linfty_empirical_average} \\
\bottomrule[1.5pt]
\end{tabular}%
}
\caption{\textcolor{greencheck}{\checkmark} denotes positive results, \textcolor{redcross}{$\boldsymbol{\times}$} impossibility results, and \textcolor{amberconj}{\textbf{?}} open problems. The assumptions are on the environment (stateful vs.\ stateless) and learner model (no-regret or Boltzmann), under the evaluation metrics of best-response distance (BR), KL divergence between policies, and $\ell_\infty$ norm.
}
\label{tab:comparison} 
\end{table}

Now that we have formalized our problem setup, we can begin to prove guarantees for various prediction strategies. In \Cref{tab:comparison}, we summarize our results. We consider each combination of learner model and performance measure for both the stateless and stateful cases and either establish a guarantee, propose a conjecture, or prove that a good guarantee is impossible. We use assumptions that closely match ones previously considered in the literature. We discuss this in detail in \Cref{sec:conclusion}.

\subsection{Minimizing the \textit{KL-Boltzmann distance}}

\subsubsection{Learner model: \textit{No-regret learner}}
\label{subsubsub:regret_imp}

Here we will show that for a no-regret learner, there does not exist a generic guarantee on the cumulative $\ell^2$ distance better than $D_{\ell_2}(R^{*}, R_{1:T}) = \Theta(T)$. Since the $\ell_2$ distance is bounded by the KL, our impossibility result also implies an impossibility result for the KL-Boltzmann distance, meaning we can't hope to minimize the KL-Boltzmann distance when observing a no-regret learner.

\begin{lemma}
\label{lemma: imp_noregret_norms}
Suppose a learner is no-regret. Given any algorithm $A$ predicting $R_{1:T}$, there exists an instance of the learner such that $D_{\ell_2}(R^*, R_{1:T}) = \Theta(T)$.
\prooflink{app:proof_lemma_imp_noregret_norms}
\end{lemma}

This shows that minimizing the $\ell_2$ distance is impossible in our case. We can now show that having a large $\ell_2$ distance implies having a large KL distance.

\begin{proposition}
\label{prop:kl_imp_noregret_human}
Suppose a learner is no-regret. Given any algorithm $A$ predicting $R_{1:T}$, there exists an instance of the learner such that $\KLBP(R^*, R_{1:T}) = \Theta(T)$.
\prooflink{app:proof_prop_kl_imp_noregret_human}
\end{proposition}

This impossibility result also immediately follows in the stateful case, since that is a strictly harder problem to solve. We can also show the same result for the $\ell_{\infty}$ distance via analogous arguments.




\begin{proposition}
\label{prop:kl_imp_noregret_human_infty}
Suppose a learner is no-regret, and we are in the stateless (or stateful) case. Given any algorithm $A$ predicting $R_{1:T}$ (or $Q_{1:T}$), there exists an instance of the learner such that $D_{\ell_{\infty}}(R^*, R_{1:T}) = \Theta(T)$ (or $D_{\ell_{\infty}}(Q^*, Q_{1:T}) = \Theta(T)$).
\prooflink{app:proof_prop_kl_imp_noregret_human_infty}
\end{proposition}

The takeaway from this section is clear: if you observe a no-regret learner, the best thing you can hope for is learning what the optimal action is. This is insufficient to achieve sub-linear prediction errors under richer evaluation measures, such as norm-based ones.

\subsubsection{Learner model: \textit{Boltzmann learner}}

What can we hope to retrieve instead when the learner is Boltzmann rational? We conjecture that the following results may hold.

\begin{conjecture}
\label{conj:statelessklbp}
    In the stateless case with a Boltzmann rational learner (with rationality parameter $\beta$), there exists an algorithm $A$ producing a sequence of reward predictions $R_{1:T}$ such that
    \begin{equation*}
        \KLBP(R^*, R_{1:T}) \leq O\left(\beta\sqrt{|\mathcal{A}|T}\right).
    \end{equation*}
\end{conjecture}

\begin{conjecture}
\label{conj:statefulklbp}
    In the stateful case with a Boltzmann rational learner, there exists an algorithm $A$ producing a sequence of action-value predictions $Q_{1:T}$ such that
    \begin{equation*}
        \KLBP(Q^*, Q_{1:T}) \leq O\left(\beta\sqrt{T}\cdot\textrm{poly}(|\mathcal{S}|, |\mathcal{A}|)\right).
    \end{equation*}
\end{conjecture}

If these conjectures hold, they would also imply the corresponding results for the best-response measure, as a consequence of \Cref{prop:lowkl_lowBR}.

\subsection{Minimizing the \textit{Best-response distance}}


\subsubsection{Learner model: \textit{No-regret learner}}

\label{subsubsec:lowregret_br}

\paragraph{Stateless case}

It is easy to show that to minimize BR if the learner is no-regret it's easy: it suffices to predict the reward function that puts all mass on the action played by the learner at the previous time step.

\begin{proposition}
\label{prop:stateless_easy_noregret}
    In the stateless case with a learner whose regret is bounded by $f(t)$ at every iteration $t$, there exists an algorithm $A$ producing a sequence of reward predictions $R_{1:T}$ such that
    \begin{equation*}
        \BRD(R^*, R_{1:T}) \leq 1 + f(T). 
        \tag*{\prooflink{app:proof_prop_stateless_easy_noregret}}
    \end{equation*}
\end{proposition}

\paragraph{Stateful case}

By running the same simple proof as above, but playing the previous action the learner played at every state, we achieve a similar guarantee. \\
\begin{corollary}
    \label{prop:stateful_easy_noregret}
    In the stateful case with a learner whose regret is bounded by $f(t)$ at every iteration $t$, there exists an algorithm $A$ producing a sequence of action-value predictions $Q_{1:T}$ such that
    \begin{equation*}
        \BRD(Q^*, Q_{1:T}) \leq \abs{S} + f(T).
    \end{equation*}
\end{corollary}

\subsubsection{Learner model: \textit{Boltzmann learner}}

We will now show that minimizing the KL-Boltzmann distance is sufficient to induce bounds on the best-response distance. Hence, while we do not have an algorithm that minimizes the KL-Boltzmann, we show that finding one would suffice to minimize the best response distance

\begin{proposition}
\label{prop:lowkl_lowBR}
Suppose that there exists an algorithm $A$ predicting $R_{1:T}$ such that $\KLBP(R^*, R_{1:T}) \leq f(T)$. Then, it follows that  $ \BRD(R^{*}, R_{1:T}) \leq O(f(T))$.
\prooflink{app:proof_prop_lowkl_lowBR}
\end{proposition}

\subsection{Minimizing the \textit{$\ell_{\infty}$ norm}}


\subsubsection{Learner model: \textit{No-regret learner}}

It is impossible to prove a meaningful guarantee in this case, for the same reason that the $\ell^2$ distance (\Cref{subsubsub:regret_imp}) fails; the same counterexample in that section applies here identically.

\subsubsection{Learner model: \textit{Boltzmann learner}}
\label{subsubsec:boltzellinf}

In the following, we will prove that for a Boltzmann-rational learner (\Cref{subsubsec:boltzhuman}), a prediction strategy that makes use of the empirical average achieves good guarantees in terms of the $\ell_\infty$ error.

\paragraph{Stateless case}

Let $R_t$, $\hat{R}_t$, and $R^{*}$ denote the predictor's, learner's, and ground-truth reward function respectively. Recall that a Boltzmann-rational learner is modeled as selecting actions via $a_t \sim \pi^{\beta}(\hat{R}_t) =: \hat{p}_t$, and is also assumed to satisfy the $f$-guarantee of \Cref{eqn:boltzmannhumanlearning} for some appropriate $f$. The prediction $R_t$ can only depend on $a_{<t} \sim \hat{p}$, where $\hat{p}(a_{<t}) \equiv \prod_{\tau=1}^{t-1} \hat{p}_{\tau}(a_{\tau})$. We will assume that $R_t$ and $R^{*}$ are $\sigma$-normalized in the following sense:

\begin{definition}[stateless $\sigma$-normalized]
We say that $R: \mathcal{A} \to \mathbb{R}$ is $\sigma$-normalized (for some $\sigma \in \mathbb{R}$) if and only if $\quad \sum_{a \in \mathcal{A}} R(a) = \sigma$
\end{definition}

The goal of this section will be proving guarantees about the averaging strategy for $R_t$, defined as:

\begin{definition}[stateless $\sigma$-averaging strategy]
\label{defn:stateless_average_predictor}
Define the averaging strategy as selecting the $\sigma$-normalized $R_t$ that satisfies
$\pi^{\beta}(R_t)(a) = \frac{N_{t-1}(a)}{t-1} =: p_t(a)$
at time step $t > 1$. Such $R_t$ exists for all $t > 1$ and is unique (by $\sigma$-normalization), with the explicit form
$$R_t(a) = \frac{1}{\beta}\left(\log p_t(a) - \frac{1}{|\mathcal{A}|} \sum_{a' \in \mathcal{A}} \log p_t(a')\right) + \frac{\sigma}{|\mathcal{A}|}$$
\end{definition}
We will find that our results are independent of $\sigma$ and hence will refer to this strategy as simply the ``averaging strategy'' from now on. We can now prove our main result in the stateless case:






\begin{proposition}
\label{prop:stateless_linfty_empirical_average}
Suppose that the predictor $R_t$ follows the $\sigma-$averaging strategy, and that the learner $\hat{R}_t$ satisfies an $f$-guarantee (\Cref{eqn:boltzmannhumanlearning}). Then, with probability $1-\epsilon$, the cumulative $\ell_{\infty}$ prediction error can be bounded as:
$$D_{\ell_{\infty}}(R^{*}, R_{t_\text{e}:T}) = \sum_{t=t_{\text{e}}}^{T} \norm{R_t - R^{*}}_{\infty} \leq  \frac{2}{\beta} \overbrace{\sum_{t=t_{\text{e}}}^{T} \frac{1}{\kappa_t} \sqrt{\frac{2\log(2|\mathcal{A}|(T-1)/\epsilon)}{t-1}}}^{\sim \, O\left(\sqrt{T \log (|\mathcal{A}|T/\epsilon)}\right)} + \sum_{t=t_{\text{e}}}^{T} \frac{1}{\kappa_t} \frac{f(t-1)}{t-1}$$
defining $t_e := \min(t: N_{t-1}(a) > 0 \; \forall \; a \in \mathcal{A})$ and $\kappa_t := \min_{a \in \mathcal{A}} \min(p_t(a), p^{*}(a))$ (where $\kappa_t > 0$ for $t \geq t_{\text{e}}$), with $f$ capturing the learning rate of a Boltzmann learner (see \Cref{subsubsec:boltzhuman}).
\prooflink{app:proof_prop_stateless_linfty_empirical_average}
\end{proposition}


\textbf{Discussion.} As an example, suppose the learner satisfies a $\sqrt{t}$-guarantee. Then the above becomes:
$$D_{\ell_{\infty}}(R^{*}, R_{t_\text{e}:T}) \leq \underbrace{O\left( \sqrt{T \log(|\mathcal{A}|T/\epsilon)}\right)}_{\text{predictor-learner error}} + \underbrace{O\left(\sqrt{T}\right)}_{\text{learner-oracle error}}$$
i.e., in this case, the error in predicting the learner's estimate ($\sum_t \norm{R_t - \hat{R}_t}_{\infty}$) dominates the error in the learner's estimate of the true reward function ($\sum_t \norm{\hat{R}_t - R^{*}}_{\infty}$). The assumption that our reward functions are $\sigma$-normalized is helpful for identifiability reasons; otherwise, the averaging strategy $R_t$ (\Cref{defn:stateless_average_predictor}) is only unique up to translation, which makes the value of measure $\norm{R_t - R^{*}}_{\infty}$ ambiguous without fixing a unique $R_t$ via $\sigma$-normalization.

\paragraph{Stateful case}

Let $Q_t$, $\hat{Q}_t$, and $Q^{*}$ denote the predictor's, learner's, and ground-truth action-value functions respectively. The learner is modeled as selecting actions via $a_t \sim \pi^{\beta}(\hat{Q}_t(s_t, \cdot)) =: \hat{p}_t(\cdot|s_t)$, and is also assumed to satisfy the $f$-guarantee of \Cref{eqn:statefulboltzmannhumanlearning}. The prediction $Q_t$ at time step $t$ can only depend on $(s_{<t}, a_{<t}) \sim \hat{\mu}$, with $\hat{\mu}$ capturing the learner-environment interaction:
$$\hat{\mu}(s_{<t}, a_{<t}) = \prod_{\tau=1}^{t-1} \mu(s_{\tau}|s_{\tau-1}, a_{\tau-1}) \hat{p}_{\tau}(a_{\tau}|s_{\tau})$$
for the environment's transition distribution $\mu(s_{\tau}|s_{\tau-1}, a_{\tau-1})$ (and $\mu(s_1|s_0, a_0) \equiv \mu(s_1)$). Similarly to the stateless case, we will assume $Q_t$ and $Q^{*}$ are $\sigma$-normalized in the following sense:
\begin{definition}[stateful $\sigma$-normalized]
We say that $Q: \mathcal{S} \times \mathcal{A} \to \mathbb{R}$ is $\sigma$-normalized (for some $\sigma: \mathcal{S} \to \mathbb{R}$) iff $Q$ satisfies
$\sum_{a \in \mathcal{A}} Q(s, a) = \sigma(s), \quad  \forall \; s \in \mathcal{S}$
\end{definition}

In the stateful case, we define the averaging strategy as follows:
\begin{definition}[stateful averaging strategy]
\label{defn:stateful_average_predictor}
Define the averaging strategy as selecting the $\sigma$-normalized $Q_t$ that satisfies
$\pi^{\beta}(Q_t)(a|s) = \frac{N_{t-1}(s, a)}{N_{t-1}(s)} =: p_t(a|s)$
at time step $t > t_{\text{e}}(s)$.
Such $Q_t$ exists for all $t > t_{\text{e}}(s)$ and is unique (by $\sigma$-normalization), with the explicit form
$$Q_t(s, a) = \frac{1}{\beta}\left(\log p_t(a|s) - \frac{1}{|\mathcal{A}|} \sum_{a' \in \mathcal{A}} \log p_t(a'|s)\right) + \frac{1}{|\mathcal{A}|}$$
\end{definition}
We will again find that our results are independent of $\sigma$. The results for the stateful case follow an identical structure to the stateless case, with analogous intermediate lemmas and a final proposition:
\begin{proposition}
\label{prop:stateful_linfty_empirical_average}
Suppose that the predictor $Q_t$ follows the averaging strategy (\Cref{defn:stateful_average_predictor}), and that the learner $\hat{Q}_t$ satisfies the $f$-guarantee (\Cref{eqn:statefulboltzmannhumanlearning}). Then, with probability $1-\epsilon$, the visit-weighted cumulative prediction error can be bounded as:
\begin{align*}
\sum_{s \in \mathcal{S}} \sum_{t=t_{\text{e}}(s)}^{T} v_t(s) \, \norm{Q_t(s, \cdot) - Q^{*}(s, \cdot)}_{\infty} &\leq  \frac{2}{\beta} \overbrace{ \sum_{s \in \mathcal{S}} \sum_{t=t_{\text{e}}(s)}^{T}\frac{1}{\kappa_t(s)} \sqrt{\frac{2\log(2|\mathcal{S}||\mathcal{A}|(T-1)/\epsilon)}{t-1}}}^{\sim \, O\left(|\mathcal{S}| \sqrt{T \log (|\mathcal{S}| |\mathcal{A}|T/\epsilon)}\right)}\\
&\qquad+ \sum_{s \in \mathcal{S}} \sum_{t=t_{\text{e}}(s)}^{T} \frac{1}{\kappa_t(s)} \frac{f(N_{t-1}(s))}{\sqrt{(t-1) N_{t-1}(s)}}
\end{align*}
where $v_t(s) := \sqrt{N_{t-1}(s)/(t-1)}$ and $t_{\text{e}}(s) := \min(t: N_{t-1}(s, a) > 0 \; \forall \; a \in \mathcal{A})$, with $f$ capturing the learning rate of a Boltzmann learner as described in \Cref{subsubsec:boltzhuman}.
\prooflink{app:proof_prop_stateful_linfty_empirical_average}
\end{proposition}


\textbf{Discussion.} Suppose that the learner learns at a rate of $f(N_t(s)) \propto \sqrt{N_{t}(s)}$. Then we have:
$$\sum_{s \in \mathcal{S}} \sum_{t=t_{\text{e}}(s)}^{T} v_t(s) \, \norm{Q_t(s, \cdot) - Q^{*}(s, \cdot)}_{\infty} \leq \underbrace{O\left(|\mathcal{S}| \sqrt{T \log(|\mathcal{S}| |\mathcal{A}| T/\epsilon)}\right)}_{\text{predictor-learner error}} + \underbrace{O\left(|\mathcal{S}| \sqrt{T}\right)}_{\text{learner-oracle error}}$$
Only states $s$ that are eventually explored during the $T$-length horizon will contribute to this cumulative error (by the definition of $t_e(s)$), i.e. we do not necessarily require our MDP to be ergodic.

\section{Conclusion}
\label{sec:conclusion}

In this paper, we have formalized the problem of learning the preferences of an agent that themselves is learning to act optimally. 
We showed that for a no-regret learner, we are unable to say anything \textit{in general} about the structure of their preferences other than what their preferred action is in each state. This is an intuitive result: we need the learner to somewhat indicate their preferences for suboptimal actions if we want to have any hope of learning their full preference structure, and no-regret does not generally guarantee this. Under alternative assumptions on the learner (e.g. Boltzmann rationality), we are able to obtain general guarantees on the $\ell_\infty$ error between the ground-truth and predicted reward function based on how often different states were visited under the learner's policy.

\textbf{Choices of learner models and evaluation measures.} 
 Making the right assumptions is fundamental if we want our insights to be relevant to practice. One ambition is to have learner models that capture the bounded rationality of humans, which we note is something that the Boltzmann rationality model does not respect: for a learner to be perfectly Boltzmann, they would need to perfectly estimate the reward/action-value function. We leave the extension of our results to such learner models for future work. 
We note that our norm-based errors require that the full preference structure is captured, as intended for our downstream application. For example, while the best-response distance is a manageable metric to optimize for, it only guarantees that we identify the best action, which might not be very useful in practice. For example, we may instead want to recover reward functions that are \textit{robust}: ones that if optimized in a different environment, would still lead to a desired behavior.

\textbf{Future work.} Plenty of interesting questions remain open. For example, can we find efficient algorithms that minimize the KL-Boltzmann distance defined in \Cref{subsubsec:klbp}? Are our bounds tight? Do our results cleanly translate to the case of stochastic rewards? What could we gain if we were able to act through the expert, and not only observe, as considered in \citet{chan2019bandit}? 

Ultimately, we believe that AI systems will be deployed in increasingly high-stakes scenarios. In such cases, humans may not be perfect actors, and developing methods that can learn from a multitude of data collected from suboptimal \& learning humans will be of increasing importance. We hope that our work lays solid foundations for the study of these challenges.

\section{Acknowledgments}
We thank Nika Hagthalab and Annie Ulichney for helpful feedback on drafts.

This work was supported by a gift from Open Philanthropy (now Coefficent Giving) to the Center for Human-Compatible AI (CHAI)
at UC Berkeley. Karim Abdel Sadek and Mark Bedaywi are supported by the Cooperative AI PhD Fellowship.

\bibliography{iclr2026_malgai}

@inproceedings{ng1999policy,
  title={Policy invariance under reward transformations: Theory and application to reward shaping},
  author={Ng, Andrew Y and Harada, Daishi and Russell, Stuart},
  booktitle={Icml},
  volume={99},
  pages={278--287},
  year={1999},
  organization={Citeseer}
}

@book{sutton2018rl,
author = {Sutton, Richard S. and Barto, Andrew G.},
title = {Reinforcement Learning: An Introduction},
year = {2018},
isbn = {0262039249},
publisher = {A Bradford Book},
address = {Cambridge, MA, USA},
}

@article{silver2016mastering,
  title={Mastering the game of Go with deep neural networks and tree search},
  author={Silver, David and Huang, Aja and Maddison, Chris J and Guez, Arthur and Sifre, Laurent and Van Den Driessche, George and Schrittwieser, Julian and Antonoglou, Ioannis and Panneershelvam, Veda and Lanctot, Marc and others},
  journal={nature},
  volume={529},
  number={7587},
  pages={484--489},
  year={2016},
  publisher={Nature Publishing Group}
}

@article{jumper2021highly,
  title={Highly accurate protein structure prediction with AlphaFold},
  author={Jumper, John and Evans, Richard and Pritzel, Alexander and Green, Tim and Figurnov, Michael and Ronneberger, Olaf and Tunyasuvunakool, Kathryn and Bates, Russ and {\v{Z}}{\'\i}dek, Augustin and Potapenko, Anna and others},
  journal={nature},
  volume={596},
  number={7873},
  pages={583--589},
  year={2021},
  publisher={Nature Publishing Group UK London}
}

@article{mnih2013playing,
  title={Playing atari with deep reinforcement learning},
  author={Mnih, Volodymyr and Kavukcuoglu, Koray and Silver, David and Graves, Alex and Antonoglou, Ioannis and Wierstra, Daan and Riedmiller, Martin},
  journal={arXiv preprint arXiv:1312.5602},
  year={2013}
}

@article{christiano2017deep,
  title={Deep reinforcement learning from human preferences},
  author={Christiano, Paul F and Leike, Jan and Brown, Tom and Martic, Miljan and Legg, Shane and Amodei, Dario},
  journal={Advances in neural information processing systems},
  volume={30},
  year={2017}
}

@article{silver2021reward,
  title={Reward is enough},
  author={Silver, David and Singh, Satinder and Precup, Doina and Sutton, Richard S},
  journal={Artificial intelligence},
  volume={299},
  pages={103535},
  year={2021},
  publisher={Elsevier}
}

@inproceedings{ng2000algorithms,
  title={Algorithms for inverse reinforcement learning.},
  author={Ng, Andrew Y and Russell, Stuart and others},
  booktitle={Icml},
  volume={1},
  number={2},
  pages={2},
  year={2000}
}

@article{ziebart2010modeling,
  title={Modeling interaction via the principle of maximum causal entropy},
  author={Ziebart, Brian D and Bagnell, J Andrew and Dey, Anind K},
  year={2010},
  publisher={Carnegie Mellon University}
}

@article{skalse2023starc,
  title={STARC: A general framework for quantifying differences between reward functions},
  author={Skalse, Joar and Farnik, Lucy and Motwani, Sumeet Ramesh and Jenner, Erik and Gleave, Adam and Abate, Alessandro},
  journal={arXiv preprint arXiv:2309.15257},
  year={2023}
}

@article{arora2012online,
  title={Online bandit learning against an adaptive adversary: from regret to policy regret},
  author={Arora, Raman and Dekel, Ofer and Tewari, Ambuj},
  journal={arXiv preprint arXiv:1206.6400},
  year={2012}
}

@article{syed2007game,
  title={A game-theoretic approach to apprenticeship learning},
  author={Syed, Umar and Schapire, Robert E},
  journal={Advances in neural information processing systems},
  volume={20},
  year={2007}
}

@inproceedings{skalse2023invariance,
  title={Invariance in policy optimisation and partial identifiability in reward learning},
  author={Skalse, Joar Max Viktor and Farrugia-Roberts, Matthew and Russell, Stuart and Abate, Alessandro and Gleave, Adam},
  booktitle={International Conference on Machine Learning},
  pages={32033--32058},
  year={2023},
  organization={PMLR}
}

@INPROCEEDINGS{chan2019bandit,
  author={Chan, Lawrence and Hadfield-Menell, Dylan and Srinivasa, Siddhartha and Dragan, Anca},
  booktitle={2019 14th ACM/IEEE International Conference on Human-Robot Interaction (HRI)},
  title={The Assistive Multi-Armed Bandit},
  year={2019},
  volume={},
  number={},
  pages={354-363},
  doi={10.1109/HRI.2019.8673234}}

@inproceedings{abbeel2004apprenticeship,
  title={Apprenticeship learning via inverse reinforcement learning},
  author={Abbeel, Pieter and Ng, Andrew Y},
  booktitle={Proceedings of the twenty-first international conference on Machine learning},
  pages={1},
  year={2004}
}

@inproceedings{evans2016learning,
  title={Learning the preferences of ignorant, inconsistent agents},
  author={Evans, Owain and Stuhlm{\"u}ller, Andreas and Goodman, Noah},
  booktitle={Proceedings of the AAAI Conference on Artificial Intelligence},
  volume={30},
  number={1},
  year={2016}
}

@inproceedings{ramachandran2007bayesian,
  title={Bayesian Inverse Reinforcement Learning.},
  author={Ramachandran, Deepak and Amir, Eyal},
  booktitle={IJCAI},
  volume={7},
  pages={2586--2591},
  year={2007}
}

@inproceedings{metelli2023towards,
  title={Towards theoretical understanding of inverse reinforcement learning},
  author={Metelli, Alberto Maria and Lazzati, Filippo and Restelli, Marcello},
  booktitle={International Conference on Machine Learning},
  pages={24555--24591},
  year={2023},
  organization={PMLR}
}

@inproceedings{russell1998learning,
  title={Learning agents for uncertain environments},
  author={Russell, Stuart},
  booktitle={Proceedings of the eleventh annual conference on Computational learning theory},
  pages={101--103},
  year={1998}
}

@article{cao2021identifiability,
  title={Identifiability in inverse reinforcement learning},
  author={Cao, Haoyang and Cohen, Samuel and Szpruch, Lukasz},
  journal={Advances in Neural Information Processing Systems},
  volume={34},
  pages={12362--12373},
  year={2021}
}

@inproceedings{metelli2021provably,
  title={Provably efficient learning of transferable rewards},
  author={Metelli, Alberto Maria and Ramponi, Giorgia and Concetti, Alessandro and Restelli, Marcello},
  booktitle={International Conference on Machine Learning},
  pages={7665--7676},
  year={2021},
  organization={PMLR}
}

@article{reddy2018you,
  title={Where do you think you're going?: Inferring beliefs about dynamics from behavior},
  author={Reddy, Sid and Dragan, Anca and Levine, Sergey},
  journal={Advances in Neural Information Processing Systems},
  volume={31},
  year={2018}
}

@article{laidlaw2021uncertain,
  title={Uncertain decisions facilitate better preference learning},
  author={Laidlaw, Cassidy and Russell, Stuart},
  journal={Advances in Neural Information Processing Systems},
  volume={34},
  pages={15070--15083},
  year={2021}
}

@article{laidlaw2022boltzmann,
  title={The boltzmann policy distribution: Accounting for systematic suboptimality in human models},
  author={Laidlaw, Cassidy and Dragan, Anca},
  journal={arXiv preprint arXiv:2204.10759},
  year={2022}
}

@book{luce1959individual,
  title={Individual choice behavior},
  author={Luce, R Duncan and others},
  volume={4},
  year={1959},
  publisher={Wiley New York}
}

@article{jeon2020reward,
  title={Reward-rational (implicit) choice: A unifying formalism for reward learning},
  author={Jeon, Hong Jun and Milli, Smitha and Dragan, Anca},
  journal={Advances in Neural Information Processing Systems},
  volume={33},
  pages={4415--4426},
  year={2020}
}

@article{gleave2020quantifying,
  title={Quantifying differences in reward functions},
  author={Gleave, Adam and Dennis, Michael and Legg, Shane and Russell, Stuart and Leike, Jan},
  journal={arXiv preprint arXiv:2006.13900},
  year={2020}
}

@article{stiennon2020learning,
  title={Learning to summarize with human feedback},
  author={Stiennon, Nisan and Ouyang, Long and Wu, Jeffrey and Ziegler, Daniel and Lowe, Ryan and Voss, Chelsea and Radford, Alec and Amodei, Dario and Christiano, Paul F},
  journal={Advances in neural information processing systems},
  volume={33},
  pages={3008--3021},
  year={2020}
}

@article{bai2022training,
  title={Training a helpful and harmless assistant with reinforcement learning from human feedback},
  author={Bai, Yuntao and Jones, Andy and Ndousse, Kamal and Askell, Amanda and Chen, Anna and DasSarma, Nova and Drain, Dawn and Fort, Stanislav and Ganguli, Deep and Henighan, Tom and others},
  journal={arXiv preprint arXiv:2204.05862},
  year={2022}
}
\bibliographystyle{iclr2026_conference}

\appendix

\section{Proofs for Main Results}
\label{app:proofs_main}

\subsection{Proof of \Cref{lemma: imp_noregret_norms}}
\label{app:proof_lemma_imp_noregret_norms}
\begin{proof}
Suppose the learner always plays action $a_t = a_1$ for all $t \in [T]$. We can see that this learner is no-regret \& optimal under at least two possible choices of ground-truth reward functions: $R^{*} = R^*_1 := \frac{1}{|\mathcal{A}|}(1,1, \dots, 1)$ and $R^{*} = R^*_2 := (1, 0, \dots, 0)$ (and in fact, under any reward function $R^*$ such that $R^*(a_1) \geq R^*(a)$ for all $a \in \mathcal{A}$). This means that in the worst-case of an adversarial choice for $R^{*}$, the best prediction strategy will always incur linear error since
\begin{align*}
\min_{R_{1:T}} \max_{R^{*}} D_{\ell_2}(R^{*}, R_{1:T}) &\geq \min_{R_{1:T}} \max_{R^{*} \in \{R_1^{*}, R_2^{*}\}} D_{\ell_2}(R^{*}, R_{1:T})\\
&\geq \frac{1}{2} \min_{R_{1:T}} \left(D_{\ell_2}(R_1^{*}, R_{1:T}) + D_{\ell_2}(R_2^{*}, R_{1:T})\right)\\
&\geq \frac{T}{2} \norm{R_1^{*} - R_2^{*}}_2 = \Theta(T)
\end{align*}
where in the third line we have made use of the triangle inequality.
\end{proof}

\subsection{Proof of \Cref{prop:kl_imp_noregret_human}}
\label{app:proof_prop_kl_imp_noregret_human}

We will need the following result to prove \Cref{prop:kl_imp_noregret_human}.
\begin{lemma}
\label{lemma:kl_bounds_l2}
Given two probability vectors $p,q$ it follows that
\[
\mathrm{KL}(p\|q)\;\ge\;\frac12\,\|p-q\|_2^2.
\]
\end{lemma}

\begin{proof}
Pinsker's inequality gives $
\mathrm{TV}(p,q)\;\le\;\sqrt{\frac{\mathrm{KL}(p\|q)}{2}},
$
where $\mathrm{TV}(p,q)=\frac12\|p-q\|_1$. Rearranging, we have $\mathrm{KL}(p\|q)\;\ge\;2\,\mathrm{TV}(p,q)^2 \;=\;\frac12\,\|p-q\|_1^2.$. Since $\ell_1 > \ell_2$ always,
\[
\mathrm{KL}(p\|q)\;\ge\;\frac12\,\|p-q\|_1^2 \;\ge\;\frac12\,\|p-q\|_2^2.
\]
as desired.
\end{proof}

\Cref{prop:kl_imp_noregret_human} then follows directly by applying \Cref{lemma:kl_bounds_l2} to the impossibility result in \Cref{lemma: imp_noregret_norms}.


\subsection{Proof of \Cref{prop:kl_imp_noregret_human_infty}}
\label{app:proof_prop_kl_imp_noregret_human_infty}
\begin{proof}
Follows directly from \Cref{lemma: imp_noregret_norms}, since the $\ell_2$ is always bounded by the $\ell_\infty$.
\end{proof}

\subsection{Proof of \Cref{prop:stateless_easy_noregret}}
\label{app:proof_prop_stateless_easy_noregret}
\begin{proof}
The idea is simple. The algorithm plays an arbitrary action $b \in \mathcal{A}$ at time step 1. Then, at each iteration, the previous action $a_{t - 1}$ is maintained. The algorithm picks the reward function $R_t(a) = \mathds{1}(a = a_{t - 1})$. Then, the best response under reward $R_t$ is to play action $a_{t - 1}$. This means regret is
    \begin{align*}
        \BRD(R^*, R_{1:T}) &= (\max_{a} R^*(a) - R^*(b)) + \sum_{t = 2}^T (\max_{a} R^*(a) - R^*(a_{t - 1})) \\
        &\leq (1 - 0) + \sum_{t = 1}^{T - 1} (\max_{a} R^*(a) - R^*(a_{t})) \\
        &\leq 1 + f(T).
    \end{align*}

    By playing the reward that puts all the weight on the previous action, we are able to match the learner's regret up to a constant additive factor of 1.
\end{proof}

\subsection{Proof of the lemma from \citet{syed2007game}}
\label{app:proof_syed2007_game_lemma}
\begin{proof}
Because $R$ is the minimax solution, it must satisfy, for all other reward functions $R'$ and all possible policies $\pi \in \Pi$,
    \begin{equation*}
        \E[R(\pi) - R(\pi_E)] \leq R'(\pi) - R'(\pi_E).
    \end{equation*}
    In particular, it must hold for the true underlying reward function $R^*$, and for the optimal policy $\pi_R = \argmax_{\pi \in \Pi} \mathbb{E}[R(\pi)]$ under $R$,
    \begin{equation*}
        \E[R(\pi_R) - R(\pi_E)] \leq R^*(\pi_R) - R^*(\pi_E).
    \end{equation*}
    By definition, the left hand side is greater than zero. Therefore,
    \begin{equation*}
        R^*(\pi_E) \geq R^*(\pi_R),
    \end{equation*}
    as claimed.
\end{proof}

\subsection{Proof of \Cref{prop:lowkl_lowBR}}
\label{app:proof_prop_lowkl_lowBR}
\begin{proof}
First, use \Cref{prop:reduction_end}, and note that we have that an algorithm that only predicts a single $R_t$ such that $T \cdot \KLBP(R^*, R_t) \leq f(T)$.
Now, let's show that if $\KLBP(R^*, R_t) \leq \epsilon$, then we can bound $\BRD(R^{*}, R_{1:T})$.

Define $a^* = \argmax_a R^*(a)$, and let $p^{*}(a) := \pi^\beta(R^*)(a)$ and $p_t(a) = \pi^{\beta}(R_t)(a)$. Note that

$$\max_{a^{*} \in \mathcal{A}} (R^{*}(a^{*}) - R^{*}(a^{R_t})) = \frac{1}{\beta} \log \frac{p^{*}(a^*)}{p^{*}(a^{R_t})}$$.

for each $t$. We note that the TV distance is bounded by the KL, which implies that
\[
TV(p_t, p^{*}) \leq \sqrt{\mathrm{KL}(p_t, p^{*})/2} = \delta
\]

It is easy to see that for all actions, $|p_t(a) - p^{*}(a)| \leq TV(p_t, p^{*}) \leq \delta $. Note that now for all actions, $p^{*}(a) \geq p_t(a) - \delta$ and  $p^{*}(a) \leq p_t(a) + \delta$. This implies by definition that $p^{*}(a^*) \leq p_t(a^{R_t}) + \delta$ and $p^{*}(a^{R_t}) \geq p_t(a^{R_t}) - \delta$. So we can note that if $p_t(a^{R_t}) > \delta$, then
$$\
\frac{p^{*}(a^*)}{p^{*}(a^{R_t})} \leq \frac{p_t(a^{R_t})+\delta}{p_t(a^{R_t})-\delta}
$$
Additionally, note that if $\delta < \frac{1}{m}$ (where $m = |\mathcal{A}|$), then
$$
\frac{p^{*}(a^*)}{p^{*}(a^{R_t})} \leq \frac{\frac{1}{m}+\delta}{\frac{1}{m}-\delta} = \frac{1 + m\delta}{1-m\delta}
$$
Thus finally inducing that

$$
\BRD(R^{*}, R_{1:T}) \leq \frac{T}{\beta} \log \frac{1 +  m\delta}{1-m\delta}
$$

This shows that if the KL is bounded, the best response distance is also bounded when we predict a reward at the end. By applying \Cref{prop:reduction_end_reverse}, we are done.
\end{proof}

\subsection{Proof of \Cref{prop:stateless_linfty_empirical_average}}
\label{app:proof_prop_stateless_linfty_empirical_average}

We will now work towards proving a bound on $D_{\ell_{\infty}}(R^{*}, R_{1:T})$ (\Cref{prop:stateless_linfty_empirical_average}), starting with two lemmas:

\begin{lemma}
\label{lem:pred_pbar_infty_bound}
Suppose that $R_t$ follows an averaging strategy (\Cref{defn:stateless_average_predictor}). Then with probability $1-\epsilon$,
\begin{equation*}
\norm{p_t - \bar{p}_t}_{\infty} < \sqrt{\frac{2\log(2|\mathcal{A}|(T-1)/\epsilon)}{t-1}}
\end{equation*}
for $t > 1$, defining the time-averaged learner policy:
\begin{equation*}
\bar{p}_t(a) := \mathbb{E}_{\hat{p}(a_{<t})}[p_t(a)] = \frac{1}{t-1} \sum_{\tau=1}^{t-1} \hat{p}_{\tau}(a)
\end{equation*}
\end{lemma}



\label{app:proof_lem_pred_pbar_infty_bound}
\begin{proof}
First note that we can write
$$p_t(a) - \bar{p}_t(a) = \frac{1}{t-1} \underbrace{\sum_{\tau=1}^{t-1} (1[a = a_{\tau}] - \hat{p}_{\tau}(a))}_{=: \, M_{t-1}^a}$$
where $M_{t-1}^a$ defines a Martingale (i.e., $\mathbb{E}_{\hat{p}}[M_{t+1}^a \, | \, M_1^a, \ldots, M_t^a] = M_t$) and also satisfies $|M_{t+1}^a - M_{t}^a| \leq 1$ for all $t$, allowing us to apply the Azuma-Hoeffding inequality:
$$\mathbb{P}_{\hat{p}(a_{<t})}(|M_{t-1}^a| \geq \delta_t) \leq 2\exp(-\delta_t^2/2(t-1))$$
$$\implies \mathbb{P}_{\hat{p}(a_{<t})}(|p_t(a) - \bar{p}_t(a)| \geq \delta_t) \leq 2\exp(-(t-1)\delta_t^2/2)$$

For an $\ell^{\infty}$ guarantee, we need this bound to hold across all $a \in \mathcal{A}$. Applying a union bound,
\begin{align*}
\mathbb{P}_{\hat{p}(a_{<t})}(|p_t(a) - \bar{p}_t(a)| < \delta_t \;\; \forall \; a \in \mathcal{A} \;\; \forall \; t \in [2, T]) &= 1 - \mathbb{P}_{\hat{p}(a_{<t})}(\exists \; a \in \mathcal{A}, \, t \in [2, T]: |p_t(a) - \bar{p}_t(a)| \geq \delta_t)\\
&\geq 1 - \sum_{a \in \mathcal{A}} \sum_{t=2}^{T} \mathbb{P}_{\hat{p}(a_{<t})}(|p_t(a) - \bar{p}_t(a)| \geq \delta_t)\\
&\geq 1- 2|\mathcal{A}| \sum_{t=2}^{T} \exp(-(t-1) \delta_t^2/2)
\end{align*}
$$\implies \mathbb{P}_{\hat{p}(a_{<t})}\left(|p_t(a) - \bar{p}_t(a)| < \sqrt{\frac{2\log(2|\mathcal{A}|(T-1)/\epsilon)}{t-1}} \;\; \forall \; a \in \mathcal{A} \;\; \forall \; t \in [2, T]\right) \geq 1-\epsilon$$
And so, with probability $1-\epsilon$:
$$\norm{p_t - \bar{p}_t}_{\infty} < \sqrt{\frac{2\log(2|\mathcal{A}|(T-1)/\epsilon)}{t-1}}$$
as required.
\end{proof}

\begin{lemma}
\label{lem:pbar_oracle_infty_bound}
Suppose that the learner's $\hat{R}_t$ satisfies the $f$-guarantee of \Cref{eqn:boltzmannhumanlearning}. Then for $t > 1$ and $\bar{p}_t$ as defined in the statement of \Cref{lem:pred_pbar_infty_bound}, we can bound
\begin{equation}
\norm{\bar{p}_t - p^{*}}_{\infty} \leq \frac{\beta}{2} \frac{f(t-1)}{t-1}
\end{equation}
\end{lemma}
\begin{proof}
Note that $\hat{p}_{\tau}(a) = \text{softmax}(\beta \hat{R}_{\tau})(a)$ and $p^{*}(a) = \text{softmax}(\beta R^{*})(a)$. Since softmax is Lipschitz continuous under $\ell^{\infty}$ with constant $1/2$, we have that
$$\norm{\hat{p}_{\tau} - p^{*}}_{\infty} \leq \frac{\beta}{2} \norm{\hat{R}_{\tau} - R^{*}}_{\infty}$$
allowing us to bound
\begin{align*}
\label{eqn:humanoracleerror}
\norm{\bar{p}_t - p^{*}}_{\infty} &\leq \frac{1}{t-1} \sum_{\tau=1}^{t-1} \norm{\hat{p}_{\tau} - p^{*}}_{\infty}\\
&= \frac{\beta}{2} \frac{1}{t-1} \sum_{\tau=1}^{t-1} \norm{\hat{R}_{\tau} - R^{*}}_{\infty}\\
&\leq \frac{\beta}{2} \frac{f(t-1)}{t-1}
\end{align*}
where we have made use of the stateless $f$-guarantee (\Cref{eqn:boltzmannhumanlearning}) which says that:
$$\sum_{t=1}^{T} \norm{\hat{R}_{t} - R^{*}}_{\infty} \leq f(T)$$
\end{proof}

We are now ready to prove \Cref{prop:stateless_linfty_empirical_average}.
\label{app:proof_lem_pbar_oracle_infty_bound}
\begin{proof}[Proof of \Cref{prop:stateless_linfty_empirical_average}]
Write $p_t(a) := \pi^{\beta}(R_t)(a)$, $\hat{p}_t(a) := \pi^{\beta}(\hat{R}_t)(a)$, and $p^{*}(a) = \pi^{\beta}(R^{*})(a)$ for the Boltzmann policies associated with the predictor, learner, and ground-truth reward functions respectively. Using the expression for $R_t$ in \Cref{defn:stateless_average_predictor} (and using an analogous expression for $R^{*}$), we can write
$$R_t(a) - R^{*}(a) = \frac{1}{\beta}\left(\log \frac{p_t(a)}{p^{*}(a)} - \frac{1}{|\mathcal{A}|} \sum_{a' \in \mathcal{A}} \log \frac{p_t(a')}{p^{*}(a')}\right)$$
which lets us write
\begin{align*}
\norm{R_t - R^{*}}_{\infty} &= \max_{a \in \mathcal{A}} \frac{1}{\beta} \Bigg|\log \frac{p_t(a)}{p^{*}(a)} - \frac{1}{|\mathcal{A}|} \sum_{a' \in \mathcal{A}} \log \frac{p_t(a')}{p^{*}(a')}\Bigg|\\
&\leq \max_{a \in \mathcal{A}} \frac{1}{\beta} \left(\Bigg|\log \frac{p_t(a)}{p^{*}(a)}\Bigg| + \Bigg|\frac{1}{|\mathcal{A}|} \sum_{a' \in \mathcal{A}} \log \frac{p_t(a')}{p^{*}(a')}\Bigg|\right)\\
&\leq \max_{a \in \mathcal{A}} \frac{2}{\beta} \Bigg|\log \frac{p_t(a)}{p^{*}(a)}\Bigg|\\
&\leq \max_{a \in \mathcal{A}} \frac{2}{\beta} \frac{|p_t(a) - p^{*}(a)|}{\min(p_t(a), p^{*}(a))}\\
&\leq \frac{2}{\kappa_t \beta} \norm{p_t - p^{*}}_{\infty}
\end{align*}
for $t \geq t_{\text{e}}$, where in the second-to-last line we have made use of
$$\left|\log\frac{p_t(a)}{p^{*}(a)}\right| \leq \frac{|p_t(a) - p^{*}(a)|}{\min(p_t(a), p^{*}(a))}$$
via the mean-value theorem. We can bound $\norm{p_t - p^{*}}_{\infty}$ by decomposing
\begin{equation}
\label{eqn:probdecomp}
p_t(a) - p^{*}(a) = \underbrace{(p_t(a) - \bar{p}_t(a))}_{\text{predictor-learner error}} + \underbrace{(\bar{p}_t(a) - p^{*}(a))}_{\text{learner-oracle error}}
\end{equation}
Using the previous lemmas, with probablity $1-\epsilon$:
\begin{align*}
\norm{R_t - R^{*}}_{\infty} &\leq \frac{2}{\kappa_t \beta} (\underbrace{\norm{p_t - \bar{p}_t}_{\infty}}_{\text{\Cref{lem:pred_pbar_infty_bound}}} + \underbrace{\norm{\bar{p}_t - p^{*}}_{\infty}}_{\text{\Cref{lem:pbar_oracle_infty_bound}}})\\
&< \frac{2}{\kappa_t \beta} \sqrt{\frac{2\log(2|\mathcal{A}|(T-1)/\epsilon)}{t-1}} + \frac{1}{\kappa_t} \frac{f(t-1)}{t-1}
\end{align*}
Summing over $t \in [t_{\text{e}}, T]$ gives us the required result.
\end{proof}

The assumption of a Boltzmann-rational learner is important for being able to prove things about $\norm{R_t - R^{*}}_{\infty}$ since it ensures that the learner will always place a non-zero probability on each action (with a fixed $\beta$), giving the predictor an opportunity to understand the value of $\hat{R}_t(a)$ (and hence $R^{*}(a)$) across all actions $a \in \mathcal{A}$. In contrast, a no-regret learner has no such guarantees about exploration, since in principle they can instantly select the optimal action $a^{*}$ and commit to it for all time, making it impossible for $R_t$ to learn about the reward associated with any action other than $a^{*}$.

In summing over $t \in [t_e, T]$, we are only considering the prediction error after the learner has already explored each action. As a result, for very large $|\mathcal{A}|$, this guarantee may lose its usefulness due to ignoring the potentially long, non-negligible duration $t < t_{\text{e}}$ when the learner is still to explore all actions. In this case, a different proof strategy may be necessary for proving guarantees on these earlier errors.

\subsection{Proof of \Cref{prop:stateful_linfty_empirical_average}}
\label{app:proof_prop_stateful_linfty_empirical_average}

The following follows a similar structure to the stateless case (\Cref{app:proof_prop_stateless_linfty_empirical_average}).
\begin{lemma}
\label{lem:stateful_pred_pbar_infty_bound}
Suppose that $Q_t$ follows an averaging strategy (\Cref{defn:stateful_average_predictor}). Then with probability $1-\epsilon$,
\begin{equation*}
\norm{p_t(\cdot|s) - \bar{p}_t(\cdot|s)}_{\infty} < \sqrt{\frac{2\log(2|\mathcal{S}||\mathcal{A}|(T-1)/\epsilon)}{N_{t-1}(s)}} \qquad \forall \; s \in \mathcal{S}: N_{t-1}(s) > 0
\end{equation*}
for $t > 1$, defining the state-wise time-averaged learner policy:
$$\bar{p}_t(a|s) := \frac{1}{N_{t-1}(s)} \sum_{i=1}^{N_{t-1}(s)} \hat{p}_{\tau_i(s)}(a|s)$$
where $\tau_{i}(s)$ is the time step corresponding to the $i$th visit of state $s$.
\end{lemma}

\begin{proof}
The proof mostly follows the structure of the stateless case (\Cref{lem:pred_pbar_infty_bound}). See that
$$p_t(a|s) - \bar{p}_t(a|s) = \frac{1}{N_{t-1}(s)} \underbrace{\sum_{i=1}^{N_{t-1}(s)} (1[a_{\tau_{i}(s)} = a] - \hat{p}_{\tau_i(s)}(a|s))}_{=: \, M_{N_{t-1}(s)}^{(s,a)}}$$
where $M_{n}^{(s, a)}$ defines a Martingale (i.e., $\mathbb{E}_{\hat{\mu}}[M_{n+1}^{(s, a)}\, | \, M_1^{(s, a)}, \ldots, M_n^{(s, a)}] = M_n^{(s, a)}$) that also satisfies $|M_{n+1}^{(s, a)} - M_{n}^{(s, a)}| \leq 1$ for all $n$, allowing us to apply the Azuma-Hoeffding inequality to find:
$$\mathbb{P}_{\hat{\mu}(s_{<t}, a_{<t})}(|p_t(a|s) - \bar{p}_t(a|s)| \geq \delta_n \, | \, N_{t-1}(s) = n) \leq 2\exp(-n\delta_n^2/2)$$
$$\implies \mathbb{P}_{\hat{\mu}(s_{<t}, a_{<t})}\left(|p_t(a|s) - \bar{p}_t(a|s)| \geq \sqrt{\frac{2\log(2|\mathcal{S}||\mathcal{A}|(T-1)/\epsilon)}{N_{t-1}(s)}} \, \Bigg| \, N_{t-1}(s) = n\right) \leq \frac{\epsilon}{|\mathcal{S}||\mathcal{A}|(T-1)}$$
Applying a union bound, we find that:
\begin{align*}
\mathbb{P}_{\hat{\mu}(s_{<t}, a_{<t})}\left(|p_t(a|s) - \bar{p}_t(a|s)| < \sqrt{\frac{2\log(2|\mathcal{S}||\mathcal{A}|(T-1)/\epsilon)}{N_{t-1}(s)}} \;\; \forall \; s \in \mathcal{S} \;\; \forall \; a \in \mathcal{A} \;\; \forall \; t \in [2, T] \, \Bigg| \, N\right) \geq 1-\epsilon
\end{align*}
As a result, with probability $1-\epsilon$,
$$\norm{p_t(\cdot|s) - \bar{p}_t(\cdot|s)}_{\infty} < \sqrt{\frac{2\log(2|\mathcal{S}||\mathcal{A}|(T-1)/\epsilon)}{N_{t-1}(s)}}$$
as required.
\end{proof}


\begin{lemma}
\label{lem:stateful_pbar_oracle_infty_bound}
Suppose that the learner's $\hat{Q}_t$ satisfies the $f$-guarantee of \Cref{eqn:statefulboltzmannhumanlearning}. Then for $t > 1$ and $\bar{p}_t$ as defined in the statement of \Cref{lem:stateful_pred_pbar_infty_bound}, we can bound
\begin{equation}
\norm{\bar{p}_t(\cdot|s) - p^{*}(\cdot|s)}_{\infty} \leq \frac{\beta}{2} \frac{f(N_{t-1}(s))}{N_{t-1}(s)} 
\end{equation}

\end{lemma}

\begin{proof}
Note that $\hat{p}_{\tau}(a|s) = \text{softmax}(\beta \hat{Q}_{\tau}(s, \cdot))(a)$ and $p^{*}(a|s) = \text{softmax}(\beta Q^{*}(s, \cdot))(a)$, and hence by the Lipschitz inequality:
$$\norm{\hat{p}_{\tau}(\cdot|s) - p^{*}(\cdot|s)}_{\infty} \leq \frac{\beta}{2} \norm{\hat{Q}_{\tau}(s, \cdot) - Q^{*}(s, \cdot)}_{\infty}$$
\begin{align}
\label{eqn:statefulhumanoracleerror}
\implies \norm{\bar{p}_t(\cdot|s) - p^{*}(\cdot|s)}_{\infty} &\leq \frac{1}{N_{t-1}(s)} \sum_{i=1}^{N_{t-1}(s)} \norm{\hat{p}_{\tau_i(s)}(\cdot|s) - p^{*}(\cdot|s)}_{\infty} \nonumber \\
&\leq \frac{\beta}{2N_{t-1}(s)} \sum_{i=1}^{N_{t-1}(s)} \norm{\hat{Q}_{\tau_i(s)}(s, \cdot) - Q^{*}(s, \cdot)}_{\infty}\\
&\leq \frac{\beta}{2}\frac{f(N_{t-1}(s))}{N_{t-1}(s)}
\end{align}
where in the final line we have used the stateful $f$-guarantee (\Cref{eqn:statefulboltzmannhumanlearning}) which tells us that:
$$\sum_{i=1}^{N_T(s)} \norm{\hat{Q}_{\tau_i(s)}(s, \cdot) - Q^{*}(s, \cdot)}_{\infty} \leq f(N_T(s))$$
\end{proof}


We are now equipped to prove \Cref{prop:stateful_linfty_empirical_average}.

\begin{proof}[Proof of \Cref{prop:stateful_linfty_empirical_average}]
Similarly to the stateless case, we will write $p_t(a|s) := \pi^{\beta}(Q_t(s, \cdot))(a)$, $\hat{p}_t(a|s) := \pi^{\beta}(\hat{Q}_t(s, \cdot))(a)$, and $p^{*}(a|s) = \pi^{\beta}(Q^{*}(s, \cdot))(a)$ for the Boltzmann policies associated with the predictor, learner, and ground-truth action-value functions respectively. Analogously to the stateless case, using the explicit expression for $Q_t$ in \Cref{defn:stateful_average_predictor} (and using an analogous expression for $Q^{*}$), we can write
\begin{align*}
\norm{Q_t(s, \cdot) - Q^{*}(s, \cdot)}_{\infty} = \max_{a \in \mathcal{A}} |Q_t(s, a) - Q^{*}(s, a)| &\leq \max_{a \in \mathcal{A}} \frac{2}{\beta}\left|\log \frac{p_t(a|s)}{p^{*}(a|s)}\right|\\
&\leq \frac{2}{\kappa_t(s) \beta} \max_{a \in \mathcal{A}} |p_t(a|s) - p^{*}(a|s)|\\
&= \frac{2}{\kappa_t(s) \beta} \norm{p_t(\cdot|s) - p^{*}(\cdot|s)}_{\infty}
\end{align*}
for any $t \geq t_{\text{e}}(s)$. Then, by decomposing
\begin{equation}
\label{eqn:statefuldiffdecomp}
p_t(a|s) - p^{*}(a|s) = (p_t(a|s) - \bar{p}_t(a|s)) + (\bar{p}_t(a|s) - p^{*}(a|s))
\end{equation}
we have that with probability $1-\epsilon$,
\begin{align*}
\norm{Q_t(s, \cdot) - Q^{*}(s, \cdot)}_{\infty} &\leq \frac{2}{\kappa_t(s) \beta} (\underbrace{\norm{p_t(\cdot|s) - \bar{p}_t(\cdot|s)}_{\infty}}_{\text{\Cref{lem:stateful_pred_pbar_infty_bound}}} + \underbrace{\norm{\bar{p}_t(\cdot|s) - p^{*}(\cdot|s)}_{\infty}}_{\text{\Cref{lem:stateful_pbar_oracle_infty_bound}}})\\
&\leq \frac{2}{\kappa_t(s) \beta} \sqrt{\frac{2\log(2|\mathcal{S}||\mathcal{A}|(T-1)/\epsilon)}{N_{t-1}(s)}} + \frac{1}{\kappa_t(s)} \frac{f(N_{t-1}(s))}{N_{t-1}(s)}
\end{align*}
Summing over $t \in [t_{\text{e}}(s), T]$ gives us the required result.
\end{proof}

\section{Robustness of the Model}
\label{app:robustness_model}
After introducing our model, one might wonder why predicting a reward/action-value function at each time step is necessary, rather than just predicting a final reward/action-value function at the end of the episode. We can prove that this doesn't matter. An algorithm that predicts a reward at the end of the interaction can be used to solve the problem of predicting a reward at each step and vice versa. \textit{Importantly, this holds regardless of the choice of learner model and the choice of distance function.}

Fix a distance function \(d\) and a function \(f:\mathbb{N}\to\mathbb{R}_{\ge 0}\). We say an algorithm \(A\) that outputs a single reward estimate \(R\) after \(T\) rounds is \(f(T)\)-no-regret if, for the ground-truth reward \(R^*\),
\[
T\, d(R, R^*) \le f(T).
\]
Likewise, an algorithm \(B\) that outputs an estimate \(R_t\) on each round \(t\in\{1,\dots,T\}\) is \(f(T)\)-no-regret if
\[
\sum_{t=1}^T d(R_t, R^*) \le f(T).
\]

Any algorithm $A$ that predicts a reward estimate at every step with regret $f(T)$ can be converted to an algorithm $B$ that predicts a reward at the end with no loss in regret, if we are okay with in expectation guarantees by simply picking a reward at random.

If we are not okay with randomized guarantees, there also exists a deterministic algorithm $B$ that holds for any convex distance function $d$, by simply averaging all of the outputted rewards. Distance functions, e.g. any norm, tend to be convex, so this is a natural assumption to be making.

\begin{proposition}
    [Per-step to final reduction]
    \label{prop:reduction_end}
    Under distance function $d$, for any learner model, if there exists an algorithm $A$ that predicts rewards at every iteration and is $f(T)$ no-regret, then there exists a randomized algorithm $B$ that predicts a single reward at the end and is $f(T)$ no-regret.

    If the distance function $d$ is convex there also exists a deterministic algorithm $B'$ that predicts a single reward at the end and is $f(T)$ no-regret.
\end{proposition}
\begin{proof}
\label{proof:perstep_tofinal}
     Suppose you have access to an algorithm that predicts a sequence of rewards $R_t$ at every iteration satisfying the guarantee
    \begin{equation*}
        \sum_{t = 1}^T d(R_t, R^*) \leq f(T).
    \end{equation*}
    We will give two algorithms that take the sequence of rewards $R_1, \dots, R_T$ and output a single reward at the end.

    Consider the algorithm that simply samples some time step uniformly at random $s \sim \mathrm{Unif}([T])$, and predicts $R_s$. The expected distance of this procedure is:
    \begin{equation*}
        \E[T d(R_s, R^*)] = T \E[d(R_s, R^*)] = T \cdot \frac{1}{T} \sum_{t = 1}^T d(R_t, R^*) = \sum_{t = 1}^T d(R_t, R^*) \leq f(T).
    \end{equation*}

    To construct a deterministic algorithm when $d$ is convex, we use Jensen's inequality. In particular, consider the algorithm that simply outputs the average of the $R_t$. This has regret:
    \begin{align*}
        f(T) &\geq \sum_{t = 1}^T d(R_t, R^*) \\
             &= T \cdot \frac{1}{T} \sum_{t = 1}^T d(R_t, R^*) \\
             &\geq T d\left(\frac{1}{T}\sum_{t = 1}^T R_t, R^*\right). \qedhere
    \end{align*}
\end{proof}

While most distance functions are convex, the best-response distance (c.f. \Cref{subsubsec:bestdist}) is not. Despite this, there still exists a deterministic reduction that suffers a regret of at most $\abs{\mathcal{A}}f(T)$ regret.
\begin{proposition}
    [Best-response per-step to final]
    \label{prop:reduction_end_reverse}
    Under the best-response distance, for any learner model, if there exists an algorithm $A$ that predicts rewards at every iteration and is $f(T)$ no-regret, then there exists a deterministic algorithm $B$ that predicts a single reward at the end and is $\abs{\mathcal{A}}f(T)$ no-regret.
\end{proposition}
\begin{proof}
\label{proof:BR_perstep_tofinal}
    Suppose you have access to an algorithm that predicts a sequence of rewards $R_t$ at every iteration satisfying the guarantee
    \begin{equation*}
        \sum_{t = 1}^T R^*(a^*) - R^*(a^{R_t}) \leq f(T),
    \end{equation*}
    where recall $a^{R_t}$ is the best action under reward $R_t$. Let $a$ be the most common best action under the sequence of rewards. Necessarily, it must be played at least $T/\abs{\mathcal{A}}$ times. Therefore,
    \begin{equation*}
        \frac{T}{\abs{\mathcal{A}}} \left(R^*(a^*) - R^*(a)\right) \leq f(T).
    \end{equation*}
    Let $R$ be the reward that places a reward of 1 on $a$ and a reward of 0 otherwise. Predicting this reward at the end results in a final best-response regret of
    \begin{equation*}
        T(R^*(a^*) - R^*(a^R)) \leq \abs{\mathcal{A}}f(T),
    \end{equation*}
    as claimed.
\end{proof}

We can go the other direction too. Any algorithm that outputs a single reward at the end suffering a regret of at most $f(T)$ can be converted to an algorithm that outputs a reward at every step by running the algorithm at every step. This suffers a regret of $\sum_{t = 1}^T f(t)/t$.

\begin{proposition}
\label{prop:final_perstep_app}
    [Final to per-step]
    Suppose there exists an algorithm $B$ that is $f(T)$ no-regret, for non-decreasing $f$, and predicts a single reward at the end. Then there exists an algorithm $A$ that is $f(T)\log(T)$ no-regret and predicts a reward at every time step.

    If regret is sublinear by a polynomial factor, i.e. $f(T) \in \bO{T^\alpha}$ for some $\alpha \in [0, 1)$, then this can be improved to $1 + \alpha^{-1}(f(T) - 1)$.
\end{proposition}
\begin{proof}
\label{proof:final_toperstep}
   Suppose that there exists an algorithm that predicts a final reward $R$ satisfying the guarantee
    \begin{equation*}
        T d(R, R^*) \leq f(T).
    \end{equation*}
    To construct a low regret sequence of rewards at every iteration that only depend on observations up until that time step, consider the procedure that just calls this algorithm at every time step $t$. By assumption $d(R_t, R^*) \leq f(t)/t$. Summing over all time steps, this means total regret is
    \begin{equation*}
        \sum_{t = 1}^T d(R_t, R^*) \leq \sum_{t = 1}^T f(t)/t \leq f(T) \sum_{t = 1}^T 1/t \leq f(T)\log(T),
    \end{equation*}
    where we use the fact that regret $f$ never decreases.

    When regret is sublinear by a polynomial factor, that is when regret is $f(T) \in \mathcal{O}(T^\alpha)$ for $\alpha < 1$,
    \begin{align*}
        \sum_{t = 1}^T f(t)/t
        &= \sum_{t = 1}^T t^{\alpha - 1} \\
        &\leq 1 + \int_1^T t^{\alpha - 1} dt \\
        &= 1 + \frac{T^\alpha - 1}{\alpha}. \qedhere
    \end{align*}
\end{proof}

\end{document}